\documentclass[conference, 10pt]{IEEEtran}

\usepackage{amsmath,graphicx,epstopdf,amsthm}%
\usepackage[normalem]{ulem}
\usepackage[caption=false,font=normalsize,labelfont=sf,textfont=sf]{subfig}
\usepackage{hyperref}
\usepackage{epstopdf} 
\usepackage{amsmath,bm}    
\usepackage{verbatim} 
\usepackage{makecell}
\usepackage{array}

\usepackage[utf8]{inputenc}     
\usepackage[english]{babel}   
\usepackage{caption}    
\usepackage{enumitem} 
\usepackage{amssymb}
\usepackage{xcolor}
\usepackage{cite}
\usepackage{lipsum}

\usepackage{booktabs,caption,siunitx}
\DeclareMathOperator*{\argmin}{arg\,min}

\usepackage[ruled,vlined,linesnumbered,noend]{algorithm2e}
\DeclareCaptionLabelFormat{lc}{\MakeLowercase{#1}~#2}
\newtheorem{fact}{Fact}
\newtheorem{assumption}{Assumption}
\newtheorem{corollary}{Corollary}
\newtheorem{lemma}{Lemma}
\newtheorem{definition}{Definition}
\newtheorem{proposition}{Proposition}
\newtheorem{theorem}{Theorem}

    {\endgroup}

\AtBeginDocument{%
  \providecommand\BibTeX{{%
    \normalfont B\kern-0.5em{\scshape i\kern-0.25em b}\kern-0.8em\TeX}}}

\begin{document}
\title{Differentially-Private Multi-Tier Federated Learning}




\author{Evan Chen\IEEEauthorrefmark{1}\IEEEauthorrefmark{3}, Frank Po-Chen Lin\IEEEauthorrefmark{1}\IEEEauthorrefmark{3}, Dong-Jun Han\IEEEauthorrefmark{2}, and Christopher G. Brinton\IEEEauthorrefmark{1} \\
\IEEEauthorblockA{\IEEEauthorrefmark{1}School of Electrical and Computer Engineering, Purdue University, West Lafayette, IN, USA \\ 
\IEEEauthorrefmark{2}Department of Computer Science and Engineering, Yonsei University, Seoul, South Korea\\
Email: \IEEEauthorrefmark{1}\{chen4388, lin1183, cgb\}@purdue.edu, \IEEEauthorrefmark{2}djh@yonsei.ac.kr
}
\thanks{This work was supported by the . }
}

\maketitle
\def\thefootnote{\IEEEauthorrefmark{3}}\footnotetext{These authors contributed equally to this work.}\def\thefootnote{\arabic{footnote}}
\begin{abstract}
  While federated learning (FL) eliminates the transmission of raw data over a network, it is still vulnerable to privacy breaches from the communicated model parameters. In this work, we propose \underline{M}ulti-Tier \underline{F}ederated Learning with \underline{M}ulti-Tier \underline{D}ifferential \underline{P}rivacy ({\tt M$^2$FDP}), a DP-enhanced FL methodology for jointly optimizing privacy and performance in hierarchical networks. One of the key concepts of {\tt M$^2$FDP} is to extend the concept of HDP towards Multi-Tier Differential Privacy (MDP), while also adapting DP noise injection at different layers of an established FL hierarchy -- edge devices, edge servers, and cloud servers -- according to the trust models within particular subnetworks. We conduct a comprehensive analysis of the convergence behavior of {\tt M$^2$FDP}, revealing conditions on parameter tuning under which the training process converges sublinearly to a finite stationarity gap that depends on the network hierarchy, trust model, and target privacy level. 
  Subsequent numerical evaluations demonstrate that {\tt M$^2$FDP} obtains substantial improvements in these metrics over baselines for different privacy budgets, and validate the impact of different system configurations.
\end{abstract}     


 

\maketitle

\section{Introduction}

\noindent The concept of privacy has significantly evolved in the digital age, particularly with regards to data sharing and utilization in machine learning (ML)~\cite{Mohammad2019}. 
Federated learning (FL) has emerged as an attractive paradigm for distributing ML over networks, as it allows for model updates to occur locally on the edge devices~\cite{wang2019adaptive,lin2021timescale}.
Nonetheless, FL is also susceptible to privacy threats: it has been shown that adversaries with access to model updates can reverse engineer attributes of device-side data~\cite{Zhu2019,Wang2019TM}. This has motivated different threads of investigation on privacy preservation within the FL framework. One common approach has been the introduction of differential privacy (DP) mechanisms into FL~\cite{Shen2022imp,Zhao2021LDP,Shi2021HDP,chandrasekaran2022HDP}, where calibrated noise is injected into the data to protect individual-level information, creating a privacy-utility tradeoff for FL.

In this work, we are interested in examining and improving the privacy-utility tradeoff for DP infusion over \textit{Multi-tier FL} (MFL) systems, where multiple layers of fog network nodes (e.g., edge servers) separate edge devices from the cloud server, and conduct intermediate model aggregations~\cite{lin2021timescale,chandrasekaran2022HDP}. These nodes offer additional flexibility into where and how DP noise injection occurs, but challenge our understanding of how DP impacts performance in MFL. Motivated by this, we investigate the following question:
\begin{center}
\textbf{\textit{What is the coupled effect between MFL system configuration and DP noise injection on ML performance?}}
\end{center}

\textbf{Related Works}: The introduction of DP into FL has traditionally followed two paradigms: (i) central DP (CDP), involving noise addition at the main server~\cite{kon2017federated,Xiong2022CDP}, and (ii) local DP (LDP), which adds noise at each edge device~\cite{Zhao2021LDP,Shen2022imp,mobi2023Qiao,Liu2023mobi}. 
CDP generally leads to a more accurate final model, but it hinges on the trustworthiness of the main server. Conversely, LDP forgoes this trust requirement but requires a higher level of noise addition at each device to compensate~\cite{naseri2022local}.

More recently, a new paradigm called hierarchical DP (HDP) has been introduced~\cite{chandrasekaran2022HDP}. 
While focusing on a three-layer network scenario, HDP assumes that certain nodes present within the network can be trusted even if the main server cannot. These nodes are entrusted with the task of adding calibrated DP noise to the aggregated models. 
Instead of injecting noise uniformly, HDP enables tailoring noise addition to varying levels of trust within the system, highlighting an opportunity for privacy amplification at the ``super-node" level.
there remains a significant gap in studies that delve into Hierarchical DP (HDP) or explore the convergence behavior in these systems, particularly those featuring hierarchical structures. 


\textbf{Main contributions}: Despite these efforts, none have yet attempt to extend the concept of HDP towards more general multi-tier networks, nor any work have attempted to rigorously characterize or optimize a system that fuses the flexible trust model of HDP with MFL training procedures. In this work, we bridge this gap through the development of \textit{\underline{M}ulti-Tier \underline{F}ederated Learning with \underline{M}ulti-Tier \underline{D}ifferential \underline{P}rivacy} ({\tt M$^2$FDP}), along with its associated theoretical analysis. Our convergence analysis reveals conditions necessary to secure robust convergence rates in DP-enhanced Multi-Tier FL systems, which also leads to the discovery that the effect of DP-protection differs on different layer of the network hierarchy. Our main contributions can be summarized as follows:

\begin{itemize}[leftmargin=5mm]
    \item We formalize {\tt M$^2$FDP}, which integrates flexible Multi-Tier DP (MDP) trust models with MFL that consists of arbitrary number of network layers (Sec.~\ref{sec:tthf}). {\tt M$^2$FDP} is designed to preserve a target privacy level throughout the entire training process, instead of only at individual aggregations, allowing for a more effective balance between privacy preservation and model performance.
    \item We characterize the convergence behavior of {\tt M$^2$FDP} under non-convex ML loss functions (Sec.~\ref{sec:convAnalysis}). Our analysis shows that with an appropriate choice of step size,
    the cumulative average global model will converge sublinearly with rate $\mathcal O(1/\sqrt{T})$. Where the DP noise injected at different tiers impacts the stationarity gap differently. 
    
    \item Through numerical evaluations, we demonstrate that {\tt M$^2$FDP} obtains substantial improvements in convergence speed and trained model accuracy relative to existing DP-based FL algorithms (Sec.~\ref{sec:experiments}). 
    Our results also corroborate the impact of the network configuration and trust model on training performance in our bounds.
    \vspace{-0.01in}
\end{itemize}


\section{Preliminaries and System Model}
\noindent This section introduces key DP concepts (Sec.~\ref{ssec:DP}), our multi-tier network (Sec.~\ref{subsec:syst1}), and the target ML task (Sec.~\ref{subsec:syst2}).

\subsection{Differential Privacy (DP)}\label{ssec:DP}
Differential privacy (DP) characterizes a randomization technique according to parameters $\epsilon, \delta$. Formally, a randomized mechanism $\mathcal M$ adheres to ($\epsilon$,$\delta$)-DP if it satisfies the following: 
\begin{definition}[($\epsilon$,$\delta$)-DP~\cite{Dwork2014DP}]
For all $\mathcal D$, $\mathcal D'$ that are adjacent datasets, and for all $\mathcal S \subseteq \text{Range}(\mathcal M)$, it holds that:
\begin{align}
    \textstyle \Pr[\mathcal M(D)\in\mathcal S]\leq e^{\epsilon}\Pr[\mathcal M(D')\in\mathcal S]+\delta,
\end{align}
where $\epsilon>0$ and $\delta\in(0,1)$. 
\end{definition}
\noindent $\epsilon$ represents the privacy budget, smaller $\epsilon$ implies a stronger privacy guarantee. $\delta$ bounds the probability of the privacy mechanism being unable to preserve the $\epsilon$-privacy guarantee, and adjacent dataset means they differ in at most one element.    

\begin{figure*}[t]
  \centering
    \includegraphics[width=0.91\textwidth]{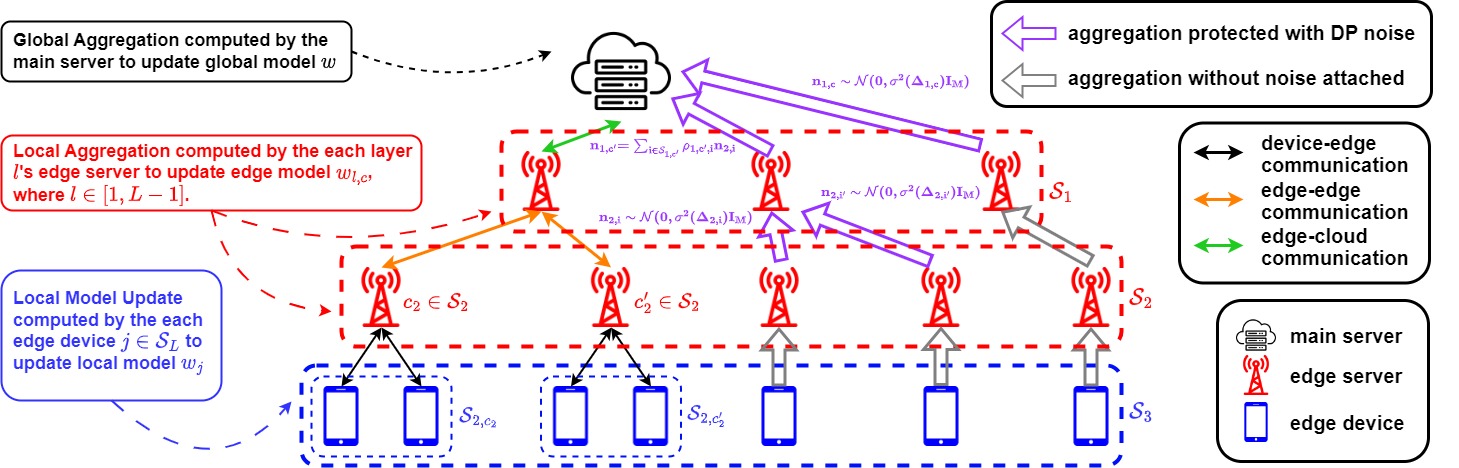}
     \caption{Multi-tier network architecture with total four layers ($L = 3$), where $\mathcal{S}_1$ and $\mathcal{S}_2$ are the set of edge servers between the devices and the main server, and $\mathcal{S}_3$ is the set of edge devices computing local gradient update. During local aggregation at each layer $l$, the noise added towards the child node's model will be a linear combination of existing noises if the child node is in the set of insecure edge servers $\mathcal{N}_{U,l+1}$. Otherwise, if the child node is in the set of secure edge servers $\mathcal{N}_{T,l+1}$, a new noise that guarantees differential privacy will be generated and injected to the child node's model.}
     \label{fig2}
     \vspace{-1.5em}
\end{figure*}
 

\textbf{Gaussian Mechanism}:
The Gaussian mechanism is a commonly employed randomization mechanism compatible with $(\epsilon,\delta)$-DP. With it, noise sampled from a Gaussian distribution is introduced to the output of the function being applied to the dataset. This function, in the case of {\tt M$^2$FDP}, is the computation of gradients. Formally, to maintain ($\epsilon$,$\delta$)-DP for any query function $f$ processed utilizing the Gaussian mechanism, the standard deviation $\sigma$ of the distribution must satisfy:
\begin{align}
    \textstyle \sigma > \frac{\sqrt{2log(1.25/\delta)} \Delta f}{\epsilon},
\end{align}
where $\Delta f$ is the $L_2$-sensitivity defined below.

\textbf{$L_2$-Sensitivity}:
The sensitivity of a function is a measure of how much the output can change between any given pair of adjacent dataset. Specifically, given that $\mathcal{D}$ and $\mathcal{D'}$ are adjacent datasets, the $L_2$-sensitivity for a function $f$ is defined as:
\begin{align}
    \textstyle\Delta f = \max_{\mathcal D, \mathcal D'} &\Vert f(\mathcal D)-f(\mathcal D')\Vert_2
\end{align}
where $\Vert\cdot\Vert_2$ is the $L_2$ norm.
In our setting, $L_2$ sensitivity allows calibrating the amount of Gaussian noise to be added to ensure the desired ($\epsilon$,$\delta$)-differential privacy in FL.

\subsection{Multi-Tier Network System Model}
\label{subsec:syst1} 

\textbf{System Architecture}: We consider the multi-tier network architecture depicted in Fig.~\ref{fig2}. We define $L$ as the number of layers below the cloud server, indicating that the total number of layers is $L+1$. The hierarchy, from bottom to top, consists of devices at the $L^\textrm{th}$ layer, $L-1$ layers of edge servers, and the cloud server. The primary responsibilities of these layers include local model training (edge devices), local aggregation (edge servers), and global aggregation (cloud server).

The set of nodes in each layer $l$ is denoted by $\mathcal{S}_l$, where $\mathcal{S}_{l,c}$ is the set of child nodes in layer $l+1$ connecting to a node $c $ in $\mathcal{S}_l$. The sets $\mathcal{S}_{l,c}$ are all pairwise disjoint, which means every node in layer $l+1$ are connected to a specific upper layer node in layer $l$. We define $N_l = |\mathcal{S}_l| = \sum_{c\in \mathcal{S}_{l-1}}|\mathcal{S}_{l-1,c}|$ as the total number of nodes in layer $l$. As a result, the total number of edge devices is $N_L =|\mathcal{S}_L|$.


\textbf{Threat Model}:
We consider the setting where any information that is being passed through a communication link has the risk of being capture and analyze, but there is no risk of any information being poisoned or replaced. This includes real world problems such as external eavesdroppers, or when certain edge servers exhibit \textit{semi-honest behavior}~\cite{Zhu2019,Wang2019TM}, where the edge servers still perform operations as the FL protocol intended, the exception being that they may seek to extract sensitive information from shared FL models.

Whenever aggregation between a node $c\in \mathcal{S}_l$ and its set of child nodes $\mathcal{S}_{l,c}$ is performed, all of the child nodes will consider whether they are willing to trust the node $c$ or not. If they trust it, then node $c$ is categorized as a (i) \textit{secure/trusted} edge server, $c \in \mathcal{N}_{T,l} \subseteq \mathcal{S}_l$, otherwise, it is being categorized as an (ii) \textit{insecure/untrusted} edge server $c\in \mathcal{N}_{U,l} \subseteq \mathcal{S}_l$. For the parent node's perspective, if any child node are labeled as insecure, $\mathcal{S}_{l,c} \cap \mathcal{N}_{U,l+1} \neq \phi$, then it will also label itself as insecure. Hence the set of insecure nodes in layer $l$ is a subset of the all child nodes of insecure nodes in layer $l-1$.

 


\subsection{Machine Learning Model} \label{subsec:syst2}
Each edge device $i\in \mathcal{S}_L$ has a dataset $\mathcal{D}_i$ comprised of $D_i=|\mathcal{D}_i|$ data points. We consider the datasets $\mathcal{D}_i$ to be non-i.i.d. across devices, as is standard in FL research.
The \textit{loss} $\ell(d; \mathbf w)$ quantifies the fit of the ML model to the learning task. It is linked to a data point $d \in \mathcal{D}_i$ and depends on the 
 model $\mathbf{w} \in \mathbb{R}^M$ (with $M$ representing the model's dimension). 
Consequently, the \textit{local loss function} for device $i$ is:
\begin{align}\label{eq:1}
\textstyle F_i(\mathbf w)=\frac{1}{D_i}\sum_{(\mathbf x,y)\in\mathcal D_i}
\ell(\mathbf x,y;\mathbf w).
\end{align}
We define the \textit{edge-level loss function} for each $ c\in \mathcal{S}_{l}$ as
\begin{align*}
\textstyle F_{l,c}(\mathbf w) = \begin{cases}
    \sum_{i\in \mathcal{S}_{l,c}} \rho_{l,c,i} F_i(\mathbf{w}), &\quad l = L-1,\\
    \sum_{j\in \mathcal{S}_{l,c}} \rho_{l,c,j} F_{l+1,j}(\mathbf{w}), &\quad l \in [1, L-2],
\end{cases}
\end{align*}
where $\rho_{l,c, i} = 1/|\mathcal{S}_{l,c}|$ symbolizes the relative weight when aggregating from layer $l+1$ to $l$.
Finally, the \textit{global loss function} is defined as the average loss across all subnets:
\begin{align*}
\textstyle F(\mathbf w)=\sum_{c=1}^{N_1} \rho_c F_{1,c}(\mathbf w),
\end{align*}
where $\rho_c = 1/N_1$ is each subnet's weight towards global loss.

The primary objective of ML model training is to pinpoint the optimal global model parameter vector $\mathbf w^* \in \mathbb{R}^M$ such that
$
\mathbf w^* = \mathop{\argmin_{\mathbf w \in \mathbb{R}^M} }F(\mathbf w).
$

\section{Proposed Methodology}
\label{sec:tthf}
\noindent In this section, we formalize {\tt M$^2$FDP}, including its operation timescales (Sec. \ref{subsec:syst3}), training process (Sec. \ref{subsec:form}), and DP mechanism (Sec. \ref{ssec:DP_main}).

\subsection{Model Training Timescales} \label{subsec:syst3}

Training in {\tt M$^2$FDP} follows a slotted-time representation. \textit{Global aggregation} is carried out by all layers of the hierarchy to collect model parameters to the main server at each time index $t = 0, 1, \cdots, T$. In between two global iterations $t$ and $t+1$, $K^t$ times of \textit{local model training iterations} are carried out by the edge devices via stochastic gradient descent (SGD).



In {\tt M$^2$FDP}, the main server first broadcasts the initial global model ${ w}^{(1)}$ to all devices at $t = 1$. For each local iteration $k$, we define $\mathcal{K}_l^t \subseteq \{1, 2, \ldots, K^t\}$ as the set of local iterations where \textit{local aggregation} from the devices toward edge servers at layer $l$ is performed. We assume the set of local aggregation iterations, $\{\mathcal{K}_1^t, \mathcal{K}_2^t, \ldots, \mathcal{K}_{L-1}^t\}$ to be pairwise disjoint, which means for each $k$, only one local aggregation will happen.
\begin{algorithm}[t]
{ \scriptsize
\caption{{\tt M$^2$FDP}: Multi-Tier Federated Learning with Multi-Tier Differential Privacy}
\label{alg:1}
\KwIn{Number of global aggregations $T$, minibatch sizes $\vert\xi_i^{(t,k)}\vert$, privacy level $\epsilon, \delta$} 
\KwOut{Global model ${ w}^{(T+1)}$}
Initialize $w^{(1)}$ and broadcast it among the edge devices through the edge servers, set $w_j^{(1,1)} = w^{(1)}, \forall j \in \mathcal{S}_L$.\\
Initialize $K^1$ to decide the number of local gradient descent to be performed before the first global aggregation.\\
Initialize $\eta^1$ according to the conditions on step size in Theorem~\ref{thm:noncvx}.

\For{$t = 1, \ldots, T$}{
\For{$k = 1, \ldots, K^t$}{
Each edge device $i \in \mathcal{S}_L$ performs a local SGD update based on \eqref{eq:SGD} using $w_i^{(t,k)}$ to obtain $\widetilde{w}_j^{(t,k)}$.\\
\uIf{$k \in \mathcal{K}^t_l, l\in [1, L-1]$}{
\For{$l' = L-1, L-2, \ldots, l$}{All edge devices/servers in layer $l'+1$ performs local aggregation towards edge servers $c$ at layer $l'$. \\
\uIf{$c$ is secure, $c\in \mathcal{N}_{T,l'}$}{
Perform aggregation based on \eqref{eq:secure_local_aggr}. }
\Else{Perform aggregation based on \eqref{eq:insecure_local_aggr}.}
}
After the local aggregation, edge servers at layer $l$ broadcast its aggregated model $w_{l, d_{l,j}}^{(t,k)}$  to edge devices $j$, $w_j^{(t,k)} = {w}_{l, d_{l,j}}^{(t,k)}, \quad j\in\mathcal{S}_L.$

}
\Else{
${{{w}}}_j^{(t,k)} = \widetilde{w}_j^{(t,k)}, \quad j\in\mathcal{S}_L.$
}

}
Performs global aggregation to compute the global model ${ w}^{(t+1)}$ on the main server based on \eqref{eq:secure_local_aggr}, \eqref{eq:insecure_local_aggr}, and \eqref{eq:glob_aggr_3}.\\
Broadcast the aggregated model to edge devices for the next round of local iterations $w_i^{(t+1,1)} = { w}^{(t+1)} \; \forall i \in \mathcal{S}_L$ .

}

}
\end{algorithm}
\setlength{\textfloatsep}{0pt}

\vspace{-2mm}
\subsection{{\tt M$^2$FDP} Training and Aggregations}\label{subsec:form}
Algorithm~\ref{alg:1} summarizes the full {\tt M$^2$FDP} procedure. The algorithm can be divided into three major stages: local model updated using on-device gradients, local model aggregation to edge servers on any itermediate layer $l \in [1, L-1]$, and global model aggregation to the main server.

\textbf{Local model update}: At local iterations $k \in [1, K^t]$, device $i$ randomly selects a mini-batch $\xi_i^{(t,k)}$ from its local dataset $\mathcal D_i$. Using this mini-batch, it calculates the stochastic gradient based on its preceding local model $ w_i^{(t,k)}$, 
${g}_{i}^{(t,k)}=\frac{1}{\vert\xi_i^{(t,k)}\vert}\sum_{(\mathbf x,y)\in\xi_i^{(t,k)}}
    \nabla\ell(\mathbf x,y; w_i^{(t,k)})$.
We assume a uniform selection probability $q$ of each data point, i.e., $q=\vert\xi_i^{(t)}\vert/D_i,~\forall i$. Device $i$ employs ${ g}_{i}^{(t,k)}$ to determine $ {\widetilde{{w}}}_i^{(t)}$:
\begin{align} \label{eq:SGD}
    \widetilde{w}_i^{(t,k+1)} = 
            w_i^{(t,k)}-\eta^t { g}_{i}^{(t,k)},~k \in [1,K^t].
\end{align}
Here, $\eta^{t} > 0$ signifies the step size. Using ${\widetilde{w}}_i^{(t,k+1)}$ as the base, the \textit{updated local model} ${w}_i^{(t,k+1)}$ is determined in one of several ways depending on the trust model, described next.

If no local aggregation is performed at time $k$, i.e., $k \notin \cup_{l=1}^{L-1}\mathcal{K}_l^t$, the updated model follows ${{{w}}}_i^{(t,k+1)} = {\widetilde{{w}}}_i^{(t,k+1)}$ in \eqref{eq:SGD}. On the other hand, if $k \in \cup_{l=1}^{L-1}\mathcal{K}_l^t$, then the updated local model inherits the local model aggregation described next.

\textbf{Local model aggregations}: When $k \in \mathcal{K}_l^t$, the network performs aggregation towards edge servers located at layer $l$. The \textit{local aggregated model} $w_{l',c}^{(t,k)}$ will first be computed at the all edge servers $c \in \mathcal{S}_{l'}$ at level $l' = L-1$, then sequentially compute towards layer $l$. Base on different types of edge servers, the following operations may be performed.

\textit{(i) Aggregation towards secure edge servers:} For any aggregation performed between layer $l'+1$ and layer $l'$, if the parent node $c\in \mathcal{S}_{l'}$ is a secure edge server, i.e., $c \in \mathcal{N}_{T,l'}$, each child node $i \in \mathcal{S}_{l',c}$ sends the server its model parameters without noise attached. Hence ${ w}_{l',c}^{(t,k)}$ is computed as follows:
\vspace{-0.05in}
\begin{align}\label{eq:secure_local_aggr}
\textstyle w_{l',c}^{(t,k)} &= \begin{cases}
    \textstyle\sum_{i\in \mathcal{S}_{l',c}}\rho_{l',c,i}\widetilde{w}_i^{(t,k)},\quad& l' = L-1,\\
    \textstyle\sum_{i\in \mathcal{S}_{l',c}}\rho_{l',c,i}{w}_{l'+1,i}^{(t,k)} \quad& l' \in [l, L-2].\\
\end{cases}
\end{align} 

\textit{(ii) Aggregation towards insecure edge servers:} Conversely, for any aggregation performed between layer $l'+1$ and layer $l'$, if the edge server $c \in \mathcal{S}_{l'}$ is considered untrustworthy, i.e., $c\in\mathcal{N}_{U,l'}$, each child node $i \in\mathcal{S}_{l',c}$ 
injects DP noise $n_{l'+1,i}^{(t,k)}$ within its transmission. 
This leads to the following adjustment from the previous aggregated model ${ w}_{l',c}^{(t,k)}$:

\vspace{-1mm}

{\small
\begin{align}\label{eq:insecure_local_aggr}
\textstyle \widetilde{w}_{l,i}^{(t,k)} &= \begin{cases}
    {w}_{l',i}^{(t,k)}, &\quad i \in \mathcal{N}_{U,l},\\
    {w}_{l',i}^{(t,k)} + n_{l',i}^{(t,k)}, &\quad i \in \mathcal{N}_{T,l},
\end{cases}\\
\textstyle w_{l',c}^{(t,k)} &= \begin{cases}
    \sum_{i\in \mathcal{S}_{l',c}}\rho_{l',c,i}\left(\widetilde{w}_i^{(t,k)} + n_{L,i}^{(t,k)}\right), \quad& l' = L-1,\\
    \sum_{i\in \mathcal{S}_{l',c}}\rho_{l',c,i}\widetilde{w}_{l'+1,i}^{(t,k)}, \quad& l' \in [l, L-2].\notag
\end{cases}
\end{align} 
}

For all edge devices $i \in \mathcal{S}_L$, the noise is sampled from:
\vspace{-1mm}
{\small
\begin{align}
    \textstyle n_{L,i}^{(t,k)} &\textstyle\sim \mathcal{N}(0, \sigma^2(\Delta_{L,i})I_M),\notag\\
    \textstyle\Delta_{L,i} &\textstyle\overset{\Delta}{=} \max_{\mathcal{D}, \mathcal{D}'}\left\|\eta^t\sum_{k=1}^{K^t}\left(g_i^{(t,k)}(\mathcal{D}) - g_i^{(t,k)}(\mathcal{D}')\right)\right\|.
\end{align}
}
\hspace{-3mm} where $\sigma(\cdot)$ is a function of $\Delta$ defined later in Proposition~\ref{prop:GM}.
For all edge servers $i \in \mathcal{N}_{T,l'}$, the noise is sampled from:

\vspace{-3mm}

{\small
\begin{align}
    \textstyle n_{l',i}^{(t,k)} &\textstyle \sim \mathcal{N}(0, \sigma^2(\Delta_{l',j})I_M)\notag\\
    \textstyle \Delta_{l',i} &\textstyle \overset{\Delta}{=} \max_{\mathcal{D}, \mathcal{D}'}\Big\|\eta^t\sum_{k=1}^{K^t}\sum_{c_{l'+1}\in \mathcal{S}_{l',i}}\rho_{l, i, c_{l'+1}}, \ldots\\
    &\textstyle \sum_{j\in \mathcal{S}_{L-1, c_{L-1}}} \rho_{L-1, c_{L-1}, j} \left(g_j^{(t,k)}(\mathcal{D}) - g_j^{(t,k)}(\mathcal{D}')\right)\Big\|.\notag
    \end{align}
}
    

Finally, after computing the local aggregated model, the edge server at layer $l$ broadcasts the final model across its subnet. We define $d_{l,j} \in \mathcal{S}_l$ to be the parent node of the edge device $j \in \mathcal{S}_L$ at layer $l$. Then the devices subsequently synchronize their local models as $ {{{w}}}_j^{(t,k)} = {w}_{l, d_{l,j}}^{(t,k)}, \forall j\in\mathcal{S}_L$.

\textbf{Global model aggregation}: 
At the end of each training interval $K^t$, a global aggregation occurs. The aggregation through edge servers follows the same procedure as \eqref{eq:secure_local_aggr}, \eqref{eq:insecure_local_aggr}.

For the aggregation from edge servers of layer $1$ towards the main server, all models will be protected. For secure edge servers $c \in \mathcal{N}_{T,1}$, noise $n_{1, c}^{(t,K^t+1)} \sim \mathcal{N}(0, \sigma^2(\Delta_{1,c})I_M)$ will be attached. For insecure edge servers $c \in \mathcal{N}_{U,1}$, no noise will be added since noises are already attached during aggregations within the subnet of edge server $c$.
\vspace{-3mm}

{\small
\begin{align}
\label{eq:glob_aggr_3}
    \widetilde{w}_{1,c}^{(t, K^t+1)} &= \textstyle\begin{cases}
        {w}_{1,c}^{(t, K^t+1)}, &\quad c \in \mathcal{N}_{U,1},\\
        {w}_{1,c}^{(t, K^t+1)} + n_{1,c}^{(t,K^t+1)}, &\quad c \in \mathcal{N}_{T,1},
    \end{cases}\notag\\
    w^{(t+1)} &= \textstyle\sum_{c = 1}^{N_1} \rho_{c} \widetilde{w}_{1,c}^{(t, K^t+1)}.
\end{align}
}
Now the resulting global model ${w}^{(t+1)}$ is employed to synchronize the local models maintained by the edge devices, i.e., ${{{w}}}_j^{(t+1,1)} = { w}^{(t+1)}, \forall j \in \mathcal{S}_L$.

\subsection{DP Mechanisms}\label{ssec:DP_main}
We now dictate the procedure for configuring the DP noise variables ${ n}_{l,c}^{(t,k)}$ for all layers $l \in [1, L]$. In this study, we focus on the Gaussian mechanisms from Sec.~\ref{ssec:DP}, though {\tt M$^2$FDP} can be adjusted to accommodate other DP mechanisms too. 
 
Following the composition rule of DP~\cite{Dwork2014DP}, we aim to take into account the privacy budget \textit{across all aggregations throughout the training}. This will ensure cumulative privacy for the complete model training process, rather than considering each individual aggregation in isolation~\cite{Shen2022imp}. Below, we define the Gaussian mechanisms, incorporating the moment accountant technique \cite{Shi2021HDP,Zhou2023HDL}.
These mechanisms utilize~\eqref{eq:grad_bound} from Assumption~\ref{assump:genLoss} which is stated in Sec.~\ref{sec:convAnalysis}.
\vspace{-1mm}
\begin{proposition}[Gaussian Mechanism~\cite{Abadi2019MA}]\label{prop:GM}
    Under Assumption~\ref{assump:genLoss}, there exists constants $c_1$ and $\alpha_l$ such that given the data sampling probability $q$ at each device, and the total number of aggregations $L$ conducted during the model training process, for any $\epsilon<c_1qL$,  {\tt M$^2$FDP} exhibits $(\epsilon,\delta)$-differential privacy for any $\delta>0$, so long as the DP noise follows $\mathbf n_{DP}^{(t,k)}\sim \mathcal{N}(0, \sigma^2(\Delta_{l,c})\mathbf{I}_M)$, where
    \vspace{-1mm}
    \begin{align}
        \textstyle\sigma(\Delta_{l,c}) = \alpha_l\frac{q\Delta_{l,c}\sqrt{L\log(1/\delta)}}{\epsilon}.
    \end{align}
    Here, $\Delta_{l,c}$ represent the $L_2$-norm sensitivity of the gradients exchanged during the aggregations towards a node $c \in \mathcal{S}_l$.
\end{proposition}
The characteristics of the DP noises introduced during local and global aggregations is established using Proposition~\ref{prop:GM}. The relevant $L_2$-norm sensitivities can be established as follows:
\begin{lemma}\label{lem:DeltaM}
    Under Assumption~\ref{assump:genLoss}, the $L_2$-norm sensitivity of the exchanged gradients during local and global aggregations can be obtained as:

\begin{itemize}
    \item For any $l \in [1, L-1]$, and any given node $c_l \in \mathcal{S}_l$, the sensitivity is bounded as
    \begin{align}\label{eq:locAgg_egd_sens1}
        &\textstyle\Delta_{l, c_l} \leq \frac{2\eta^t K^t G}{\prod_{l'=l}^{L-1} s_{l'}}, \quad \forall l \in [1, L-1].
    \end{align}
    \item For any given node $j \in \mathcal{S}_L$, the sensitivity is bounded as
    \begin{align}\label{eq:locAgg_egd_sens2}
        &\textstyle\Delta_{L,j} \leq 2\eta^t K^t G.
    \end{align}
\end{itemize}
\end{lemma}
By defining $K^{\max} = \max_{t\in [1,T]}K^t$ as the largest local training interval throughout the whole training process, the variance of the Gaussian noises
${\sigma^2}(\Delta_{l,c})$ can be determined based on Proposition~\ref{prop:GM} and Lemma~\ref{lem:DeltaM} by setting $L=K^{\max}T$.

\section{Convergence Analysis} \label{sec:convAnalysis}
\subsection{Analysis Assumptions}
We first establish a few general and commonly employed assumptions that we will consider throughout our analysis.

\begin{assumption}[Characteristics of Noise in SGD~\cite{lin2021timescale,Shen2022imp,Zhang2022AS}] \label{assump:SGD_noise}
    Consider ${\mathbf n}_{i}^{(t)}=\widehat{\mathbf g}_{i}^{(t)}-\nabla F_i(\mathbf w_{i}^{(t)})$ as the noise of the gradient estimate through the SGD process for device $i$ at time $t$. The noise variance is upper bounded as $\mathbb{E}_t[\Vert{\mathbf n}_{i}^{(t)}\Vert^2]\leq \sigma^2~\forall i,t$.
\end{assumption}
\begin{assumption}[General Characteristics of Loss Functions \cite{Shen2022imp}, \cite{Zhang2022AS}]\label{assump:genLoss} 
Assumptions applied to loss functions include:
\begin{itemize}[leftmargin=5mm]
    \item \textbf{Bounded gradient}: The stochastic gradient norm of the loss function $\ell(\cdot)$ is bounded by a constant $G$, i.e.,  
    \begin{align} \label{eq:grad_bound}
            \textstyle\Vert\widehat{\mathbf g}_{i}^{(t)}\Vert \leq  G, ~\forall i, t.
     \end{align}
     \item  \textbf{Smoothness}: Each local loss $F_i$ is $\beta$-smooth $\forall i\in\mathcal{I}$, i.e., 
{
    \begin{equation}\label{eq:11_beta}
        \textstyle\Vert \textstyle\nabla F_i(\mathbf w_1)-\nabla F_i(\mathbf w_2)\Vert \leq \beta\Vert \mathbf w_1-\mathbf w_2 \Vert, ~\forall \mathbf w_1, \mathbf w_2 \in \mathbb{R}^M,
        \hspace{-1.mm}
    \end{equation} }
\end{itemize}
\end{assumption}

\subsection{Preliminary Quantities}
Before proceeding to our main result in Sec.~\ref{ssec:convAvg}, we establish a few quantities to facilitate our analysis. Since during aggregation towards an insecure server, whether the model comes from an secure or insecure determines the level of incoming noise, we define $p_{l,c}$ to be the ratio of child nodes for some $c \in \mathcal{N}_{U,l-1}$ that are secure servers
$\textstyle p_{l,c} \overset{\Delta}{=} \frac{|\mathcal{N}_{T, l} \cap \mathcal{S}_{l-1, c}|}{|\mathcal{S}_{l-1, c}|},  l \in [1, L-1]$.

We then define the maximum and minimum ratio on each layer $l \in [1, L-1]$ as $ p_l^{\max} = \max_{c \in \mathcal{N}_{U,l-1}} p_{l,c}$ and $ p_l^{\min} = \min_{c \in \mathcal{N}_{U,l-1}} p_{l,c}$. To simplify the representation of our result, we further define $p_0^{\max} = p_0^{\min} = 0$ and $p_L^{\max} = p_L^{\min} = 1$. Finally, we define $s_l = \min_{c\in \mathcal{S}_{l}} |S_{l,c}|$ to be the size of the smallest set of child nodes out of all edge servers $c \in \mathcal{S}_l$:

\subsection{General Convergence Behavior of {\tt M$^2$FDP}}
\label{ssec:convAvg}
We now present our main theoretical result, that the cumulative average of the global loss gradient can attain sublinear convergence to a controllable region around a stationary point.

\begin{theorem} \label{thm:noncvx}
        Under Assumptions~\ref{assump:SGD_noise} and \ref{assump:genLoss}, if $\eta^t=\frac{\gamma}{\sqrt{t+1}}$ with $\gamma\leq\min\{\frac{1}{K^{\max}},\frac{1}{T}\}/\beta$, the cumulative average of global gradient satisfies
{\small
\begin{align} 
        &\textstyle\frac{1}{T}\sum_{t=1}^T \left\|\nabla F(w^{(t)})\right\|^2 \leq \underbrace{\textstyle\frac{2\beta F(w^{(1)})}{\sqrt{T+1}} }_\textrm{(a$_1$)} + \underbrace{\textstyle\frac{K^{\max}\left(G^2\left(1 + \frac{1}{\beta}\right) + \sigma^2\right)}{T}}_\textrm{(a$_2$)}\label{eq:noncvx_rate}\\
        &\textstyle+ \underbrace{\textstyle\frac{8 L M (K^{\max})^4 q^2\log(1/\delta)}{\epsilon^2}\sum_{l=1}^{L}(1 - p_{l-1}^{\min})^2 \bigg(\mathcal{A}_l + \mathcal{B}_l + \mathcal{C}_l\bigg)}_\textrm{(b)},\notag
\end{align}
}
 where
{\small
\begin{align}
\label{eq:ABC_terms}
        \mathcal{A}_l &=  \textstyle p_{l}^{\max}\frac{\alpha_l^2}{\prod_{l' = l}^{L-1} s_{l'}^2}, \notag\\
        \mathcal{B}_l &= \textstyle\sum_{m = l}^{L-1} \prod_{l' = l}^{m-1}(1 - p_{l'}^{\min})p_{m}^{\max}\frac{\alpha_m^2}{\prod_{l' = l}^{ m-1} s_{l'}\prod_{l'' = m}^{ L-1} s_{l''}^2}, \notag\\
        \mathcal{C}_l &= \textstyle\prod_{l' = l}^{ L-1}(1 - p_{l'}^{\min})\frac{\alpha_L^2}{\prod_{l' = l}^{ L-1} s_{l'}}.
\end{align}
}
\end{theorem}

Noise injection creates a delicate balance between privacy conservation and model performance: as the number of global aggregations ({\small$T$}) increases, ($a_1$) and ($a_2$) in~\eqref{eq:noncvx_rate} decrease, while the overall noise level in (b) escalates. 
As suggested by Proposition~\ref{prop:GM} and Lemma~\ref{lem:DeltaM}, the variance of the DP noise inserted should scale with the total count of global aggregations, $T$. To counterbalance the influence of DP noise accumulation over successive aggregations, we enforce the condition $\eta^t\leq 1/T$ to scale down the DP noise by a factor of $T$. This strategy steers the bound in~\eqref{eq:noncvx_rate} towards the region denoted by (b), rather than allowing for constant amplification. At the same time, this highlights the trade-off between privacy preservation and model performance: although the condition $\eta^t\leq1/T$ serves to reduce DP noise, it simultaneously results in a smaller learning rate which slows {\tt M$^2$FDP} training.  

In addition, (b) conveys the beneficial influence of secure edge servers. Each term $\mathcal{A}_l$, $\mathcal{B}_l$, and $\mathcal{C}_l$ represents how the layer where noises are injected effects local aggregation towards layer $l$. $\mathcal{A}_l$ represents the part of the hierarchy where no noises were injected before aggregating towards servers at level $l$, The noise level introduced during aggregations is reduced by a factor of $1/\prod_{l' = l}^{L-1} s_{l'}^2$. This is contrast to the factor of $1/\prod_{l' = l}^{L-1} s_{l'}$ for term $\mathcal{C}_l$, which represents the part of the hierarchy where noises are injected when the edge servers in $\mathcal{S}_L$ aggregates towards edge servers at layer $L-1$. This noise reduction of an additional factor of $\prod_{l' = l}^{L-1} s_{l'}$ underscores the rationale behile {\tt M$^2$-FDP}'s integration of HDP with HFL, allowing for an effective reduction in the requisite DP noise for preserving a given privacy level. 

The term $\mathcal{B}_l$ represents the part of the hierarchy where the noises are injected some layer between the edge devices and the final local aggregation layer $l$. The layer $m-1$ in $\mathcal{B}_l$ is the layer where the noise is been injected. We can observe that the factor $1/\prod_{l' = l}^{ m-1} s_{l'}\prod_{l'' = m}^{ L-1} s_{l''}^2$ can be splitted into two parts: (i) The effect of all layers below $m$, have the same factor $\prod_{l'' = m}^{ L-1} s_{l''}^2 $ as in $\mathcal{A}_l$. (ii) The effect of all layers above $m$, have the same factor $\prod_{l' = l}^{ m-1} s_{l'}$ as in $\mathcal{C}_l$. Indicating the \textit{noises injected to the lower hierarchy layers, inflicted larger harm towards the accuracy. This shows a trade-off between protection over a HFL algorithm and the training accuracy.}

We introduce the following corollary to demonstrate this result in a more intuitve way.

\begin{corollary} \label{cor:noncvx}
    Under Assumptions~\ref{assump:SGD_noise} and \ref{assump:genLoss}, if $\eta^t=\frac{\gamma}{\sqrt{t+1}}$ with $\gamma\leq\min\{\frac{1}{K^{\max}},\frac{1}{T}\}/\beta$, and let $m$ be the layer where all edge servers below it are all secure servers ($m$ is the lowest layer to have insecure servers), i.e. $ p_{l'}^{\max} = p_{l'}^{\min} = 1, \forall l' \in [m+1 , L]$, then the cumulative average of global gradient satisfies
    
    \vspace{-3mm}
{\small
    \begin{align} \label{eq:cor_noncvx}
        &\textstyle\frac{1}{T}\sum_{t=1}^T \left\|\nabla F(w^{(t)})\right\|^2 \leq \textstyle\frac{2\beta F(w^{(1)})}{\sqrt{T+1}}+ \frac{K^{\max}\left(G^2\left(1 + \frac{1}{\beta}\right) + \sigma^2\right)}{T}\notag\\
        &\textstyle+ \frac{8 L M (K^{\max})^4 q^2\log(\frac{1}{\delta})}{\epsilon^2}\sum_{l=1}^{m}  \frac{(1 - p_{l-1}^{\min})^2p_l^{\max}\alpha_l^2}{\prod_{l'=l}^{L-1}s_{l'}^2} \\
        &\textstyle+ \underbrace{\textstyle\frac{8 L M (K^{\max})^4 q^2\log(\frac{1}{\delta})}{\epsilon^2}\sum_{l=1}^{m} \frac{(1 - p_{l-1}^{\min})^2(1 - p_l^{\min})(\alpha_l')^2}{\prod_{l'=l}^{m}s_{l'}\prod_{l''=m+1}^{L-1}s_{l''}^2},}_\textrm{(c)}\notag
\end{align}
}
\hspace{-3.5mm} where $a_l{'}$ is a constant that is a linear combination of $\alpha_{l+1}, \alpha_{l+2} \ldots, \alpha_L$.
\end{corollary}
 The magnitude of term (c) is determined by $m$. When $m$ is higher, which means noises are only injected at higher hierarchy edge servers, the effect of DP noise is decreases. When $m$ is lower, the effect of DP noise on accuracy increases. \textit{This shows that the priority of ensuring trustworthy on higher level hierarchy is different to lower level hierarchy.}

\section{Experimental Evaluation}
\label{sec:experiments}



\subsection{Simulation Setup}
\label{ssec:setup} 
By default, we consider a network with three layers ($L = 2$), and $N_2 = 50$ edge devices evenly distributed across $N_1 = 10$ subnets. We use two commonly employed datasets for image classification tasks: Fashion-MNIST (F-MNIST) and CIFAR-10.  
Following prior work~\cite{wang2019adaptive,lin2021timescale}, the training samples from each dataset are distributed across the edge devices in a non-i.i.d manner, in which each device exclusively contains datapoints from $3$ out of $10$ labels. 
For each dataset, we consider training a 12-layer convolutional neural network (CNN) accompanied by a softmax function and cross-entropy loss. The model dimension is set to $M=7840$. This evaluation setup provides insight into {\tt M$^2$FDP}'s performance when handling non-convex loss functions.
Also, unless otherwise stated, we assume $p_1 = 0.5$ for edge servers located of layer $l=1$, and that semi-honest entities are all governed by the same total privacy budget $\epsilon=1, \delta=10^{-5}$.

\begin{figure}[t]
\includegraphics[width=0.48\textwidth,height=0.15\textheight]{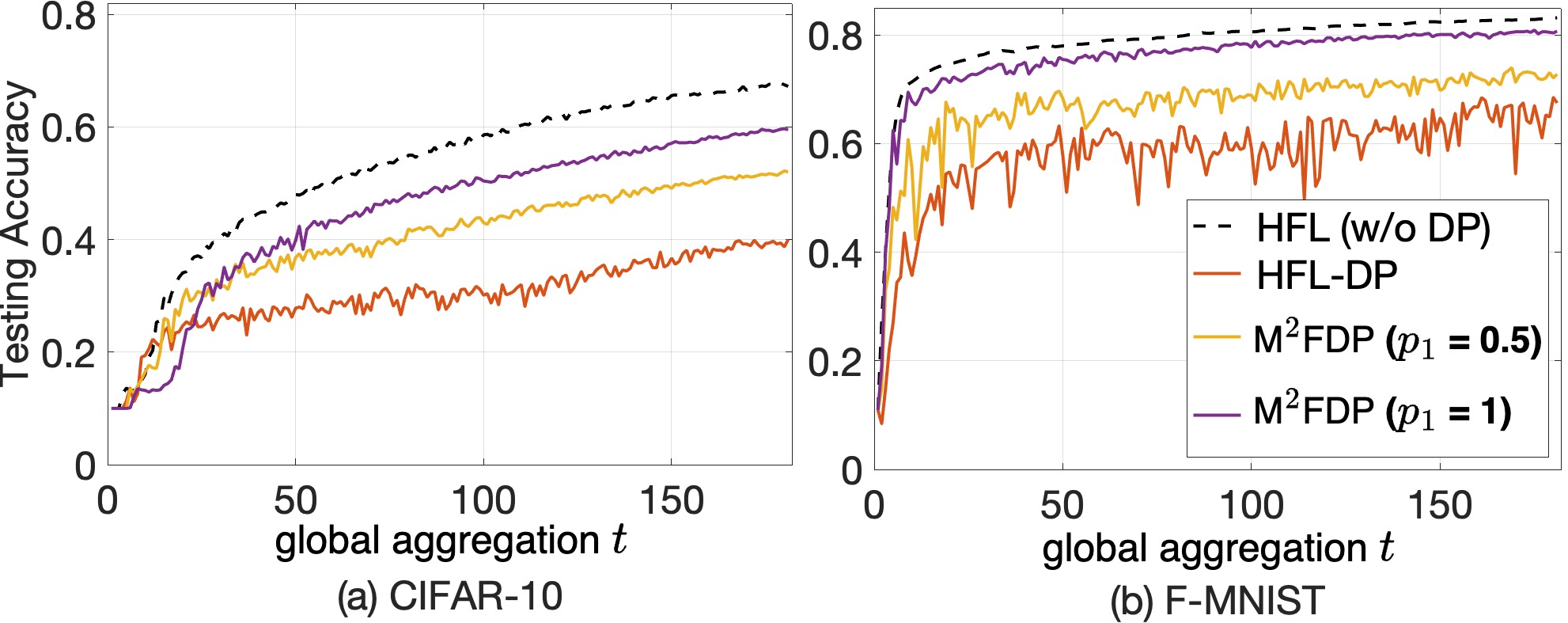}
\centering
\vspace{-0.05in}
\caption{Performance comparison between {\tt M$^2$FDP}, the {\tt HFL-DP} baseline from~\cite{Shi2021HDP}, and an upper bound established by hierarchical {\tt FedAvg} without DP. {\tt M$^2$FDP} significantly outperforms {\tt HFL-DP} and is able to leverage trusted edge servers effectively.}
\label{fig:mnist_poc_1_all}
\end{figure} 

\vspace{-0.2mm}
\subsection{{\tt M$^2$FDP} Comparison to Baselines}
\label{ssec:conv-eval}
Our first experiments examine the performance of {\tt M$^2$FDP} compared with baselines. We utilize the conventional hierarchical {\tt FedAvg} algorithm~\cite{liu2020client}, which offers no explicit privacy protection, as our upper bound on achievable accuracy (labeled {\tt HFL (w/o DP)}). We also implement {\tt HFL-DP}~\cite{Shi2021HDP}, which employs LDP within HFL, for competitive analysis. Further, we consider {\tt PEDPFL}~\cite{Shen2022imp}, for which we assume the edge devices all form a single subnet and apply LDP.


Fig.~\ref{fig:mnist_poc_1_all} demonstrates results comparing {\tt M$^2$FDP} without control to baselines. Each algorithm employs a local training interval of $K^t = 20$ and conduct local aggregations after every five local SGD iterations. We see {\tt M$^2$FDP} obtains performance enhancements over {\tt HFL-DP} by exploiting secure edge servers in the hierarchical architecture's middle-layer. This improvement is observed both in terms of superior accuracy and decreased accuracy perturbation as the ratio of secure edge server ($p_1$) increases. 
Notably, compared to the upper bound benchmark, {\tt M$^2$FDP} with $p_1=1$ achieves an accuracy within $8\%$ and $3\%$ of the benchmark for CIFAR-10 and F-MNIST, respectively.
The exploitation of secure edge servers into the middle-layers of the hierarchy significantly mitigates the amount of noise required to maintain an equivalent privacy level.

\subsection{Impact of System Parameters}

\subsubsection{Portion of Secure Edge Servers} \label{subsubsec:combiner}
We next consider the impact of the probability $p_1$ of a subnset having a secure edge server. Fig.~\ref{fig:eps_vs} shows that {\tt M$^2$FDP} obtains a considerable enhancement in privacy-performance tradeoff as the probability escalates. Specifically, under the same privacy conditions, {\tt M$^2$FDP} exhibits an improvement of at least $20\%$ for CIFAR-10 and $10\%$ for F-MNIST when all the edge servers in the mid-layer can be trusted (i.e., $p_1=1$) compared to {\tt HFL-DP} (i.e., $p_1 = 0$).

\begin{figure}[t]
\includegraphics[width=0.49\textwidth]{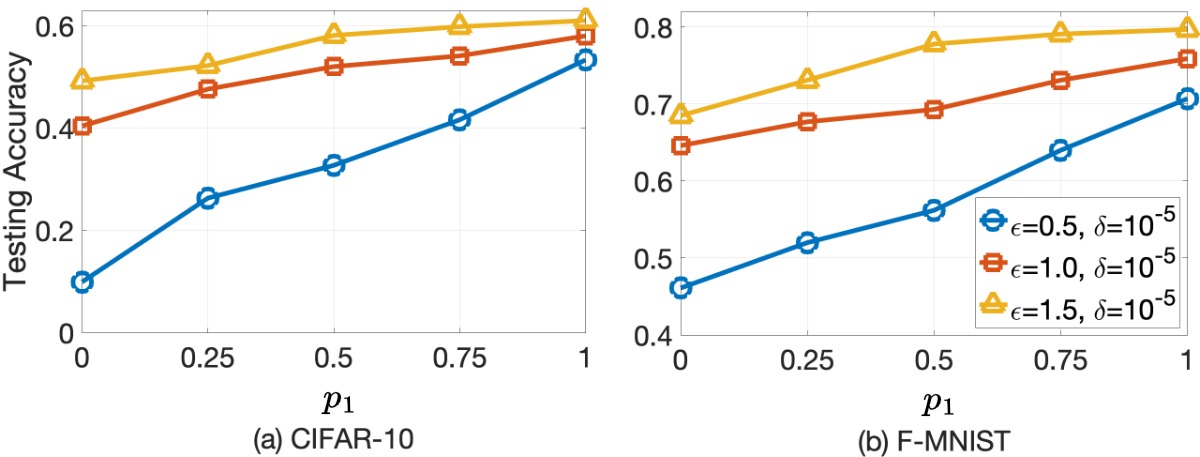}
\centering
\vspace{-0.2in}
\caption{Interplay between privacy and performance in {\tt M$^2$FDP} across various probabilities ($p_1$) of a subnet's linkage to a secure edge server under different privacy budgets ($\epsilon$).} 
\label{fig:eps_vs} 
\vspace{-4mm} 
\end{figure}

\subsubsection{Varying Network Configurations} \label{subsubsec:netSize}
Next, we investigate the impact of different network configurations on the performance of {\tt M$^2$FDP}. Two distinct configurations are evaluated: (i) Config. 1, wherein the size of subnets ($|\mathcal{S}_{1,c}| = s_1$) is kept at $5$ as the number of subnets ($N_1$) increases from $2$ to $10$; (ii) Config. 2, wherein $N_1$ is fixed at $2$, and $s_1$ increases from $5$ to $25$.
Fig.~\ref{fig:mnist_poc_2_all} gives the results. The positive correlation between network size and model performance is apparent in both configurations.
This can be attributed to the diminishing contribution of individual devices in aggregations.
This observation aligns well with Theorem~\ref{thm:noncvx}, which quantifies the noise reduction as the network size expands.

\begin{figure}[t]
\includegraphics[width=0.49\textwidth]{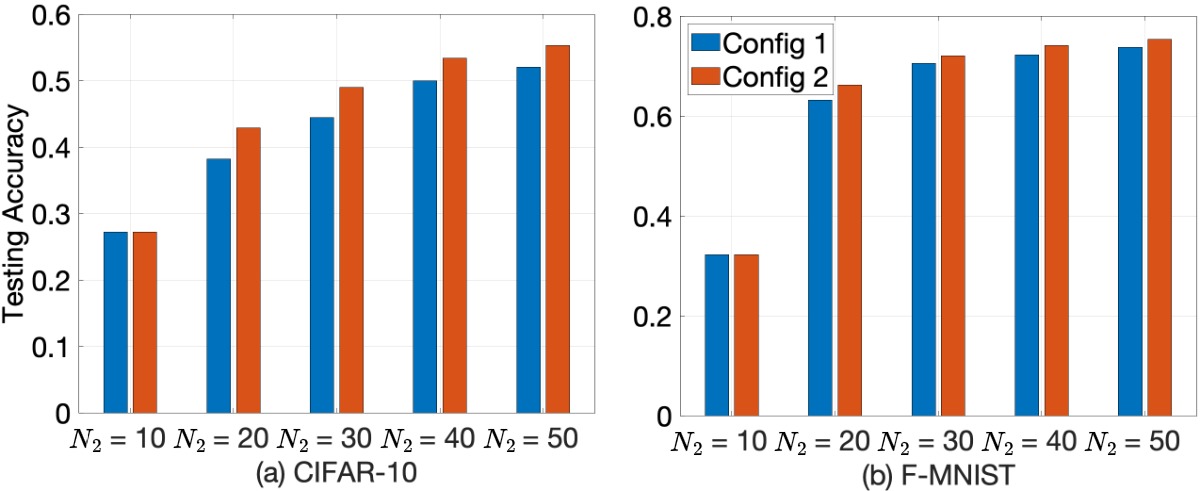} 
\centering
\vspace{-0.2in}
\caption{Impact of various network configurations on the performance of {\tt M$^2$FDP}. Under the same network size, enhancing the size of each subnet $s_1$ yields superior test accuracy compared to merely increasing the number of subnets $N_1$.}
\label{fig:mnist_poc_2_all}
\end{figure}
Also, Config. 2 leads to superior model performance compared to Config. 1: even while retaining the same network size, increasing subnet sizes is more beneficial than increasing the number of subnets. This pattern once again aligns with Theorem~\ref{thm:noncvx}: when subnets are linked to a secure edge server, the extra noise needed to maintain an equivalent privacy level can be downscaled by $1/s_1^2$.


\section{Conclusion}
\noindent In this study, we developed {\tt M$^2$FDP}, which integrates multi-tier differential privacy (MDP) into multi-tier federated learning (MFL) to enhance the trade-off between privacy and performance. We conducted a thorough theoretical analysis of {\tt M$^2$FDP}, identifying conditions under which the algorithm will converge sublinearly to a controllable region around a stationary point, and revealing the impact of different system factors on the privacy-utility trade-off. Numerical evaluations confirmed {\tt M$^2$FDP}'s superior training performance and improvements in resource efficiency compared to existing DP-infused FL/HFL algorithms.      

\bibliographystyle{IEEEtran}
\bibliography{sample-base}
 

\newpage

\begingroup
\onecolumn
\raggedbottom

\appendix
\small
\setcounter{lemma}{0}
\setcounter{theorem}{0}
\setcounter{assumption}{0}
\setcounter{corollary}{0}
\setcounter{section}{0}
\section*{Introduction to Notations and Preliminaries used in the Proofs}\label{app:notations}
\noindent  We first define several auxiliary notations that we will use in the proof:

We define the maximum $L2$-norm sensitivity from each layer's aggregation as:
\begin{equation}
    \Delta_{l} \overset{\Delta}{=} \max_{c_l \in \mathcal{S}_l} \Delta_{l, c_l}, \quad \forall l \in [1, L].
\end{equation}

We define the node $d_{l,j}$ to be the parent node of a leaf node $j \in \mathcal{S}_L$ at a specific layer $l$ (which is a unique node in a hierarchical structure).

\subsection{Proof of Theorem \ref{thm:noncvx} and Corollary \ref{cor:noncvx}}
\label{app:thm2}
\begin{theorem} \label{thm:noncvx} 
        Under Assumptions~\ref{assump:SGD_noise} and~\ref{assump:genLoss}, upon using {\tt DP-HFL} for ML model training, if $\eta^t=\frac{\gamma}{\sqrt{t+1}}$ with $\gamma\leq\min\{\frac{1}{K^{\max}},\frac{1}{T}\}/\beta$, the cumulative average of global loss gradients satisfies
\begin{align*} 
        &\frac{1}{T}\sum_{t=1}^T \left\|\nabla F(w^{(t)})\right\|^2 \leq \frac{2\beta}{\sqrt{T+1}} \mathbb{E}\left[ F(w^{(1)}) -  F(w^{(T+1)})\right] + \frac{K^{\max}\left(G^2\left(1 + \frac{1}{\beta}\right) + \sigma^2\right)}{T}\\
        &+ \frac{8 L M (K^{\max})^4 q^2\log(1/\delta)}{\epsilon^2}\sum_{l=1}^{L}(1 - p_{l-1}^{\min})^2 \bigg(p_{l}^{\max}\frac{\alpha_l^2}{\prod_{l' = l}^{L-1} s_{l'}^2} + \sum_{m = l}^{L-1} \prod_{l' = l}^{m-1}(1 - p_{l'}^{\min})p_{m}^{\max}\frac{\alpha_m^2}{\prod_{l' = l}^{ m-1} s_{l'}\prod_{l'' = m}^{ L-1} s_{l''}^2} \\&+ \prod_{l' = l}^{ L-1}(1 - p_{l'}^{\min})\frac{\alpha_L^2}{\prod_{l' = l}^{ L-1} s_{l'}}\bigg)
\end{align*}
\end{theorem}

\begin{proof}
    Consider $t \in [1, T]$, for a given aggregation interval $K^t$, the global average of the local models follows the following dynamics:
    \begin{align}
    \label{global_w_iter}
        w^{(t+1)} =& \sum_{c = 1}^{N_1} \rho_c \left(w_{1,c}^{(t,K^t+1)} + n_{1,c}^{(t,K^t+1)}\right) \nonumber \\
        =& w^{(t)} - \eta^t \sum_{k=1}^{K^t} \sum_{c_1 = 1}^{N_1} \rho_{c_1} \sum_{c_2\in \mathcal{S}_{1,c_1}} \rho_{1,c_1,c_2}\ldots\sum_{j\in\mathcal{S}_{L-1,c_{L-1}}} \rho_{L-1, c_{L-1},j} {g}_j^{(t,k)}\\
        &+ \underbrace{\sum_{c=1}^{N_1}\rho_c n_{1,c}^{(t,K^t+1)} + \sum_{l=1}^{L-1} \sum_{\substack{k=1,\\ K_l^t | k }}^{K^t} \left(\sum_{c_1 = 1}^{N_1} \rho_{c_1} \sum_{c_2\in \mathcal{S}_{1,c_1}} \rho_{1,c_1,c_2}\ldots\sum_{c _l\in \mathcal{S}_l \cap \mathcal{N}_{U,l}} \sum_{i\in \mathcal{S}_{l,c_l}} \rho_{l,c_l,i}n_{l+1,i}^{(t,k)}\right)}_\text{$n_{DP}^t$}
    \end{align}
    On the other hand, the $\beta$-smoothness of the global function $F$ implies
    \begin{equation}
        F(w^{(t+1)}) \leq F(w^{(t)}) + \nabla F(w)^\top (w^{(t+1)} - w^{(t)}) + \frac{\beta}{2}\|w^{(t+1)} - w^{(t)}\|^2.
    \end{equation}
    By injecting this inequality into \eqref{global_w_iter}, taking the expectations on both sides of the inequality, and use the fact that $\mathbb{E}[n_{DP}^t] = 0$ yields:
    \begin{align}
    \label{eq:A_main_iter}
        &\mathbb{E}\left[ F(w^{(t+1)}) -  F(w^{(t)})\right] \leq \underbrace{- \eta^t \sum_{k=1}^{K^t} \mathbb{E}\left\langle \nabla F(w^{(t)}),  \sum_{c_1 = 1}^{N_1} \rho_{c_1} \sum_{c_2\in \mathcal{S}_{1,c_1}} \rho_{1,c_1,c_2}\ldots\sum_{j\in\mathcal{S}_{L-1,c_{L-1}}} \rho_{L-1, c_{L-1},j} \nabla F_j(w_j^{(t,k)})\right\rangle}_\text{(a)} \notag\\
        &+ \underbrace{\frac{(\eta^t)^2 \beta}{2} \mathbb{E}\left[\left\|\sum_{k=1}^{K^t} \sum_{c_1 = 1}^{N_1} \rho_{c_1} \sum_{c_2\in \mathcal{S}_{1,c_1}} \rho_{1,c_1,c_2}\ldots\sum_{j\in\mathcal{S}_{L-1,c_{L-1}}} \rho_{L-1, c_{L-1},j} \nabla F_j(w_j^{(t,k)})\right\|^2\right]}_\text{(b)} + \underbrace{\frac{\beta}{2}\mathbb{E}\left[\left\|n_{DP}^t\right\|^2\right]}_\text{(c)} + \frac{(\eta^tK^t)^2\beta \sigma^2}{2} 
    \end{align}
    To bound (a), we apply Lemma~\ref{lem:inner} (see Appendix \ref{app:lemmas}) to get
    \begin{align}
    \label{eq:A_bound_on_innerprod}
       &- \eta^t \sum_{k=1}^{K^t} \mathbb{E}\left\langle \nabla F(w^{(t)}),  \sum_{c_1 = 1}^{N_1} \rho_{c_1} \sum_{c_2\in \mathcal{S}_{1,c_1}} \rho_{1,c_1,c_2}\ldots\sum_{j\in\mathcal{S}_{L-1,c_{L-1}}} \rho_{L-1, c_{L-1},j} \nabla F_j(w_j^{(t,k)})\right\rangle
        \leq
        -\frac{\eta^t K^t}{2}\left\|\nabla F(w^{(t)})\right\|^2 \notag\\
        &-\frac{\eta^t}{2} \sum_{k=1}^{K^t}\left\|\sum_{c_1 = 1}^{N_1} \rho_{c_1} \sum_{c_2\in \mathcal{S}_{1,c_1}} \rho_{1,c_1,c_2}\ldots\sum_{j\in\mathcal{S}_{L-1,c_{L-1}}} \rho_{L-1, c_{L-1},j} \nabla F_j(w_j^{(t,k)})\right\|^2 + (\eta^t)^3(K^t)^3\beta^2G^2\notag\\
        &+\frac{2(\eta^t)^2 L\beta^2 M (K^t)^5 T q^2 G^2\log(1/\delta)}{\epsilon^2}\sum_{l=2}^{L}(1 - p_{l-1}^{\min})^2 \bigg(p_{l}^{\max}\frac{\alpha_l^2}{\prod_{l' = l}^{L-1} s_{l'}^2} \\&+ \sum_{m = l}^{L-1} \prod_{l' = l}^{m-1}(1 - p_{l'}^{\min})p_{m}^{\max}\frac{\alpha_m^2}{\prod_{l' = l}^{ m-1} s_{l'}\prod_{l'' = m}^{ L-1} s_{l''}^2} + \prod_{l' = l}^{ L-1}(1 - p_{l'}^{\min})\frac{\alpha_L^2}{\prod_{l' = l}^{ L-1} s_{l'}}\bigg)
    \end{align}
    To bound (b), we use Assumption~\ref{assump:genLoss} and get
    \begin{align}
    \label{eq:A_bound_on_grad}
        &\frac{(\eta^t)^2 \beta}{2} \mathbb{E}\left[\left\|\sum_{k=1}^{K^t} \sum_{c_1 = 1}^{N_1} \rho_{c_1} \sum_{c_2\in \mathcal{S}_{1,c_1}} \rho_{1,c_1,c_2}\ldots\sum_{j\in\mathcal{S}_{L-1,c_{L-1}}} \rho_{L-1, c_{L-1},j} \nabla F_j(w_j^{(t,k)})\right\|^2\right] \\
        &\leq \frac{(\eta^t)^2 \beta}{2} K^t\sum_{k=1}^{K^t} \sum_{c_1 = 1}^{N_1} \rho_{c_1} \sum_{c_2\in \mathcal{S}_{1,c_1}} \rho_{1,c_1,c_2}\ldots\sum_{j\in\mathcal{S}_{L-1,c_{L-1}}} \rho_{L-1, c_{L-1},j}\mathbb{E}\left[\left\| \nabla F_j(w_j^{(t,k)})\right\|^2\right]\\
        &\leq \frac{(\eta^tK^t)^2 \beta G^2}{2}
    \end{align}
    To bound (c), we apply Lemma~\ref{lem:bd_on_DP_noise} and get
    \begin{align}
    \label{eq:A_bound_on_DP}
        &\frac{\beta}{2}\mathbb{E}\left[\left\|n_{DP}^t\right\|^2\right] \leq \frac{\beta}{2}\mathbb{E}\left[\left\| \sum_{c=1}^{N_1}\rho_c n_{1,c}^{(t,K^t+1)} + \sum_{l=1}^{L-1} (1 - p_{l}^{\min})\sum_{\substack{k=1,\\ K_l^t | k }}^{K^t} \left(\sum_{c_1 = 1}^{N_1} \rho_{c_1} \sum_{c_2\in \mathcal{S}_{1,c_1}} \rho_{1,c_1,c_2}\ldots\sum_{c _l\in \mathcal{S}_l \cap \mathcal{N}_{U,l}} \sum_{i\in \mathcal{S}_{l,c_l}} \rho_{l,c_l,i}n_{l+1,i}^{(t,k)}\right)\right\|^2\right]\\
        & \leq \frac{\beta L}{2}\mathbb{E}\left[\left\|\sum_{c=1}^{N_1}\rho_c n_{1,c}^{(t,K^t+1)}\right\|^2\right] + \frac{\beta L}{2}\sum_{l=1}^{L-1}(1 - p_{l}^{\min})^2\mathbb{E} \left[\left\| \sum_{\substack{k=1,\\ K_l^t | k }}^{K^t} \left(\sum_{c_1 = 1}^{N_1} \rho_{c_1} \sum_{c_2\in \mathcal{S}_{1,c_1}} \rho_{1,c_1,c_2}\ldots\sum_{c _l\in \mathcal{S}_l \cap \mathcal{N}_{U,l}} \sum_{i\in \mathcal{S}_{l,c_l}} \rho_{l,c_l,i}n_{l+1,i}^{(t,k)}\right)\right\|^2\right]\\
        &\leq \frac{ 2\eta^t L\beta  M (K^t)^4 T q^2G^2\log(1/\delta)}{\epsilon^t}\sum_{l=1}^{L}(1 - p_{l-1}^{\min})^2 \bigg(p_{l}^{\max}\frac{\alpha_l^2}{\prod_{l' = l}^{L-1} s_{l'}^2} \\&+ \sum_{m = l}^{L-1} \prod_{l' = l}^{m-1}(1 - p_{l'}^{\min})p_{m}^{\max}\frac{\alpha_m^2}{\prod_{l' = l}^{ m-1} s_{l'}\prod_{l'' = m}^{ L-1} s_{l''}^2} + \prod_{l' = l}^{ L-1}(1 - p_{l'}^{\min})\frac{\alpha_L^2}{\prod_{l' = l}^{ L-1} s_{l'}}\bigg)
    \end{align}
    By substituting \eqref{eq:A_bound_on_innerprod}, \eqref{eq:A_bound_on_grad}, \eqref{eq:A_bound_on_DP} into \eqref{eq:A_main_iter}, and using Assumption~\ref{assump:genLoss} yields
    \begin{align}
        &\mathbb{E}\left[ F(w^{(t+1)}) -  F(w^{(t)})\right] \leq -\frac{\eta^t K^t}{2}\left\|\nabla F(w^{(t)})\right\|^2 + \frac{(\eta^2K^t)^2 \beta}{2} \left(G^2 + \sigma^2 + \eta^2K^t\beta G^2\right)\\
        &+\frac{2(\eta^t)^2 L\beta^2 M (K^t)^5 T q^2 G^2\log(1/\delta)}{\epsilon^2}\sum_{l=2}^{L}(1 - p_{l-1}^{\min})^2 \bigg(p_{l}^{\max}\frac{\alpha_l^2}{\prod_{l' = l}^{L-1} s_{l'}^2} + \sum_{m = l}^{L-1} \prod_{l' = l}^{m-1}(1 - p_{l'}^{\min})p_{m}^{\max}\frac{\alpha_m^2}{\prod_{l' = l}^{ m-1} s_{l'}\prod_{l'' = m}^{ L-1} s_{l''}^2} \\&+ \prod_{l' = l}^{ L-1}(1 - p_{l'}^{\min})\frac{\alpha_L^2}{\prod_{l' = l}^{ L-1} s_{l'}}\bigg)\\
        &+\frac{ 2\eta^t L\beta  M (K^t)^4 T q^2G^2\log(1/\delta)}{\epsilon^t}\sum_{l=1}^{L}(1 - p_{l-1}^{\min})^2 \bigg(p_{l}^{\max}\frac{\alpha_l^2}{\prod_{l' = l}^{L-1} s_{l'}^2} + \sum_{m = l}^{L-1} \prod_{l' = l}^{m-1}(1 - p_{l'}^{\min})p_{m}^{\max}\frac{\alpha_m^2}{\prod_{l' = l}^{ m-1} s_{l'}\prod_{l'' = m}^{ L-1} s_{l''}^2} \\&+ \prod_{l' = l}^{ L-1}(1 - p_{l'}^{\min})\frac{\alpha_L^2}{\prod_{l' = l}^{ L-1} s_{l'}}\bigg)
    \end{align}
    Using the fact that $\eta^t \leq \frac{1}{\max \{T, K^t\}\beta}$ gives us
    \begin{align}
        &\frac{\eta^t K^t}{2}\left\|\nabla F(w^{(t)})\right\|^2 \leq \mathbb{E}\left[ F(w^{(t)}) -  F(w^{(t+1)})\right] + \frac{(\eta^2K^t)^2 \beta}{2} \left(G^2 + \sigma^2 +  \frac{G^2}{K^t\beta}\right)\\
        &+\frac{4\eta^t L\beta  M (K^t)^4 T q^2\log(1/\delta)}{\epsilon^2}\sum_{l=1}^{L}(1 - p_{l-1}^{\min})^2 \bigg(p_{l}^{\max}\frac{\alpha_l^2}{\prod_{l' = l}^{L-1} s_{l'}^2} + \sum_{m = l}^{L-1} \prod_{l' = l}^{m-1}(1 - p_{l'}^{\min})p_{m}^{\max}\frac{\alpha_m^2}{\prod_{l' = l}^{ m-1} s_{l'}\prod_{l'' = m}^{ L-1} s_{l''}^2} \\&+ \prod_{l' = l}^{ L-1}(1 - p_{l'}^{\min})\frac{\alpha_L^2}{\prod_{l' = l}^{ L-1} s_{l'}}\bigg)
    \end{align}
    Dividing both hand sides by $\frac{\eta^tK^t}{2}$ and plug in $K^{\max} = \max_t K^t$, averaging across global aggregation yields
    \begin{align}
        &\frac{1}{T}\sum_{t=1}^T \left\|\nabla F(w^{(t)})\right\|^2 \leq \frac{2\beta}{\sqrt{T+1}} \mathbb{E}\left[ F(w^{(1)}) -  F(w^{(T+1)})\right] + \frac{K^{\max}\left(G^2\left(1 + \frac{1}{\beta}\right) + \sigma^2\right)}{T}\\
        &+ \frac{8 L M (K^{\max})^4 q^2\log(1/\delta)}{\epsilon^2}\sum_{l=1}^{L}(1 - p_{l-1}^{\min})^2 \bigg(p_{l}^{\max}\frac{\alpha_l^2}{\prod_{l' = l}^{L-1} s_{l'}^2} + \sum_{m = l}^{L-1} \prod_{l' = l}^{m-1}(1 - p_{l'}^{\min})p_{m}^{\max}\frac{\alpha_m^2}{\prod_{l' = l}^{ m-1} s_{l'}\prod_{l'' = m}^{ L-1} s_{l''}^2} \\&+ \prod_{l' = l}^{ L-1}(1 - p_{l'}^{\min})\frac{\alpha_L^2}{\prod_{l' = l}^{ L-1} s_{l'}}\bigg)
    \end{align}
    
\end{proof}
\begin{corollary} \label{cor:noncvx}
    Under Assumptions~\ref{assump:SGD_noise} and \ref{assump:genLoss}, if $\eta^t=\frac{\gamma}{\sqrt{t+1}}$ with $\gamma\leq\min\{\frac{1}{K^{\max}},\frac{1}{T}\}/\beta$, and let $m$ be the layer where all edge servers below it are all secure servers ($m$ is the lowest layer to have insecure servers), i.e.
    \begin{align}
    \label{eq:cor_condition}
        p_{l'}^{\max} = p_{l'}^{\min} = 1, \quad l' \in [m+1 , L],
    \end{align}
    then the cumulative average of global gradient satisfies
    \begin{align} \label{eq:cor_noncvx}
\small
        &\frac{1}{T}\sum_{t=1}^T \left\|\nabla F(w^{(t)})\right\|^2 \leq \frac{2\beta}{\sqrt{T+1}} \mathbb{E}\left[ F(w^{(1)}) -  F(w^{(T+1)})\right] \notag\\
        &+ \frac{K^{\max}\left(G^2\left(1 + \frac{1}{\beta}\right) + \sigma^2\right)}{T}\notag\\
        &+ \frac{8 L M (K^{\max})^4 q^2\log(1/\delta)}{\epsilon^2}\sum_{l=1}^{m}  \frac{(1 - p_{l-1}^{\min})^2p_l^{\max}\alpha_l^2}{\prod_{l'=l}^{L-1}s_{l'}^2}\notag\\
        &+ \frac{8 L M (K^{\max})^4 q^2\log(1/\delta)}{\epsilon^2}\sum_{l=1}^{m} \frac{(1 - p_{l-1}^{\min})^2(1 - p_l^{\min})(\alpha_l')^2}{\prod_{l'=l}^{m}s_{l'}\prod_{l''=m+1}^{L-1}s_{l''}^2},
\end{align}
\end{corollary}
\begin{proof}
    Starting from the result from Theorem~\ref{thm:noncvx}, we can inject the condition from \eqref{eq:cor_condition}:
    \begin{align}
        &\sum_{l=1}^{L}(1 - p_{l-1}^{\min})^2 \bigg(p_{l}^{\max}\frac{\alpha_l^2}{\prod_{l' = l}^{L-1} s_{l'}^2} + \sum_{m' = l}^{L-1} \prod_{l' = l}^{m'-1}(1 - p_{l'}^{\min})p_{m'}^{\max}\frac{\alpha_{m'}^2}{\prod_{l' = l}^{ m'-1} s_{l'}\prod_{l'' = m'}^{ L-1} s_{l''}^2} + \prod_{l' = l}^{ L-1}(1 - p_{l'}^{\min})\frac{\alpha_L^2}{\prod_{l' = l}^{ L-1} s_{l'}}\bigg)\notag\\
        & = \sum_{l=1}^{m}(1 - p_{l-1}^{\min})^2 \bigg(p_{l}^{\max}\frac{\alpha_l^2}{\prod_{l' = l}^{m} s_{l'}^2} + \sum_{m' = l}^{m} \prod_{l' = l}^{m'}(1 - p_{l'}^{\min})p_{m'+1}^{\max}\frac{\alpha_{m'}^2}{\prod_{l' = l}^{ m'} s_{l'}\prod_{l'' = m'+1}^{ L-1} s_{l''}^2}\bigg)\notag\\
        &\leq \sum_{l=1}^{m}(1 - p_{l-1}^{\min})^2 \bigg(p_{l}^{\max}\frac{\alpha_l^2}{\prod_{l' = l}^{m} s_{l'}^2} + \sum_{m' = l}^{m} \prod_{l' = l}^{m'}(1 - p_{l'}^{\min})p_{m'+1}^{\max}\frac{\alpha_m^2}{\prod_{l' = l}^{ m} s_{l'}\prod_{l'' = m+1}^{ L-1} s_{l''}^2}\bigg) \notag\\
        &\leq \sum_{l=1}^{m}(1 - p_{l-1}^{\min})^2 \bigg(p_{l}^{\max}\frac{\alpha_l^2}{\prod_{l' = l}^{m} s_{l'}^2} + (1-p_{l}^{\min})\frac{(\alpha_l')^2}{\prod_{l' = l}^{ m} s_{l'}\prod_{l'' = m+1}^{ L-1} s_{l''}^2}\bigg)
    \end{align}
    Plugging this inequality back to Theorem~\ref{thm:noncvx} yields the result.
\end{proof}

\pagebreak

\subsection{Lemmas and Auxiliary Results}\label{app:lemmas}
To improve the tractability of the proofs, we provide a set of lemmas in the following, which will be used to obtain the main results of the paper.
\begin{lemma}\label{lem:Delta}
    Under Assumption~\ref{assump:genLoss}, the $L_2$-norm sensitivity of the exchanged gradients during local aggregations can be obtained as follows:
\begin{itemize}
    \item For any $l \in [1, L-1]$, and any given node $c_l \in \mathcal{S}_l$, the sensitivity is bounded by
    \begin{align}\label{eq:locAgg_egd_sens1}
        \Delta_{l, c_l} = \max_{\mathcal{D}, \mathcal{D}'}\left\|\eta^t\sum_{k=1}^{K^t}\sum_{c_{l+1}\in \mathcal{S}_{l,c_l}}\rho_{l, c_l, c_{l+1}} \ldots \sum_{j\in \mathcal{S}_{L-1, c_{L-1}}} \rho_{L-1, c_{L-1}, j} \left(g_j^{(t,k)}(\mathcal{D}) - g_j^{(t,k)}(\mathcal{D}')\right)\right\| \leq \frac{2\eta^t K^t G}{\prod_{l'=l}^{L-1} s_{l'}}, \quad \forall l \in [1, L-1]
    \end{align}
    \item For any given node $j \in \mathcal{S}_L$, the sensitivity is bounded by
    \begin{align}\label{eq:locAgg_egd_sens2}
        \Delta_{L,j} = \max_{\mathcal{D}, \mathcal{D}'}\left\|\eta^t\sum_{k=1}^{K^t}\left(g_j^{(t,k)}(\mathcal{D}) - g_j^{(t,k)}(\mathcal{D}')\right)\right\| \leq 2\eta^t K^t G
    \end{align}
\end{itemize}

\end{lemma}

    \begin{proof}
    (Case $l \in [1, L-1]$) From the condition that $\mathcal{D}$ and $\mathcal{D}'$ are adjacent dataset, which means the data samples only differs in one entry. As a result, all the summations within the norm, other than the sum for $k$, can be removed directly after upper bounding $\rho_{l, c_l, c_{l+1}}$. From previous definitions, we have:
    \begin{equation}
        \rho_{l, c_l, c_{l+1}} \leq \frac{1}{s_l}
    \end{equation}
As as result, we can bound $\Delta_{l, c_l}$ as
\begin{align}
    \Delta_{l, c_l} &\leq \Delta_{l}\\
    &\leq \frac{\eta^t}{\prod_{l'=l}^{L-1} s_{l'}} \sum_{k=1}^{K^t}\max_{\mathcal{D}, \mathcal{D}'}\left\|g_j^{(t,k)}(\mathcal{D}) - g_j^{(t,k)}(\mathcal{D}')\right\|\\
    &\leq \frac{\eta^t}{\prod_{l'=l}^{L-1} s_{l'}} \sum_{k=1}^{K^t}\max_{\mathcal{D}, \mathcal{D}'}\left\|g_j^{(t,k)}(\mathcal{D}) \right\|+\left\| g_j^{(t,k)}(\mathcal{D}')\right\|\\
    &\leq \frac{2\eta^t K^t G}{\prod_{l'=l}^{L-1} s_{l'}}
\end{align}
giving us the result in~\eqref{eq:locAgg_egd_sens1}.
    
(Case $l = L$) Similarly, for the leaf nodes, we can bound $\Delta_{L, j}$ as
        \begin{align}
    \Delta_{L, j} &\leq \Delta_{L}\\
    &\leq \eta^t \sum_{k=1}^{K^t}\max_{\mathcal{D}, \mathcal{D}'}\left\|g_j^{(t,k)}(\mathcal{D}) - g_j^{(t,k)}(\mathcal{D}')\right\|\\
    &\leq \eta^t \sum_{k=1}^{K^t}\max_{\mathcal{D}, \mathcal{D}'}\left\|g_j^{(t,k)}(\mathcal{D}) \right\|+\left\| g_j^{(t,k)}(\mathcal{D}')\right\|\\
    &\leq 2\eta^t K^t G
\end{align} giving us the result in~\eqref{eq:locAgg_egd_sens2}.
    \end{proof}

\begin{lemma}\label{lem:bd_on_DP_noise}
For any given $k \in [1, K^t]$, and any layer $l \in [1, L-1]$, the DP noise that is aggregated at local iteration $k$ towards the $l^{\textrm{th}}$ layer can be bounded by:
\begin{align}
   &\mathbb{E}\left\|\sum_{i\in \mathcal{S}_{l,d_{l,i}}} \rho_{l,d_{l,i},i}n_{l+1,i}^{(t,k)} \right\|^2 \notag \\
   &\leq \frac{2\eta^t M (K^t)^2 T q^2G\log(1/\delta)}{\epsilon^2} \left(p_{l+1}^{\max}\frac{\alpha_l^2}{\prod_{l' = l+1}^{L-1} s_{l'}^2} + \sum_{m = l+1}^{L-1} \prod_{l' = l+1}^{m-1}(1 - p_{l'}^{\min})p_{m}^{\max}\frac{\alpha_m^2}{\prod_{l' = l+1}^{ m-1} s_{l'}\prod_{l'' = m}^{ L-1} s_{l''}^2} + \prod_{l' = l+1}^{ L-1}(1 - p_{l'}^{\min})\frac{\alpha_L^2}{\prod_{l' = l+1}^{ L-1} s_{l'}}\right)
\end{align}
\end{lemma}
\begin{proof}
    We first move the norm into the sum
    \begin{align}
    \label{eq:sum_of_dp_noise}
        \mathbb{E}\left\|\sum_{i\in \mathcal{S}_{l,d_{l,i}}} \rho_{l,d_{i,j},i}n_{l+1,i}^{(t,k)} \right\|^2 \leq \sum_{i\in \mathcal{S}_{l,d_{l,i}}} \rho_{l,d_{i,j},i} \mathbb{E}\left\|n_{l+1,i}^{(t,k)}\right\|^2 
    \end{align}
    Since each $\mathbb{E}\left\|n_{l+1,i}^{(t,k)}\right\|^2$ is the variance of a guassian noise that's a linear combination of multiple noises, by applying the total law of expectation:
    \begin{align}
        &\mathbb{E}\left\|n_{l+1,i}^{(t,k)}\right\|^2  \\
        &=  \mathbb E\Big[\Big\Vert{ n}_{l+1,i}^{(t,k)}\Big\Vert^2\Big|{ n}_{l+1,i}^{(t,k)}\sim \mathcal{N}(0, ({\sigma}_{l+1,i}^{(t,k)})^2)\Big]
    \\
    &= M\left({\sigma}_{l+1,i}^{(t,k)}\right)^2\\
    &\leq p_{l+1}^{\max} M\sigma^2 (\Delta_{l+1}) + (1 - p_{l+1}^{\min})p_{l+2}^{\max} \frac{M\sigma^2 (\Delta_{l+2})}{s_{l+1}} + (1 - p_{l+1}^{\min})(1 - p_{l+2}^{\min})p_{l+3}^{\max} \frac{M\sigma^2 (\Delta_{l+3})}{s_{l+1}s_{l+2}} + \ldots \prod_{l' = l+1}^{ L-1}(1-p_{l'}^{\min}) \frac{M\sigma^2(\Delta_{L})}{\prod_{l' = l+1}^{ L-1} s_{l'}}\\
    &\leq \frac{2\eta^t M (K^t)^2 T q^2G\log(1/\delta)}{\epsilon^2} \left(p_{l+1}^{\max}\frac{\alpha_l^2}{\prod_{l' = l+1}^{L-1} s_{l'}^2} + \sum_{m = l+1}^{L-1} \prod_{l' = l+1}^{m-1}(1 - p_{l'}^{\min})p_{m}^{\max}\frac{\alpha_m^2}{\prod_{l' = l+1}^{ m-1} s_{l'}\prod_{l'' = m}^{ L-1} s_{l''}^2} + \prod_{l' = l+1}^{ L-1}(1 - p_{l'}^{\min})\frac{\alpha_L^2}{\prod_{l' = l+1}^{ L-1} s_{l'}}\right)
    \end{align}
    We can plug this inequality back to \eqref{eq:sum_of_dp_noise} and obtain the results.
\end{proof}

\begin{lemma} \label{lem:inner}
Under Assumption~\ref{assump:genLoss}, we have
    \begin{align*}
       &- \eta^t \sum_{k=1}^{K^t} \mathbb{E}\left\langle \nabla F(w^{(t)}),  \sum_{c_1 = 1}^{N_1} \rho_{c_1} \sum_{c_2\in \mathcal{S}_{1,c_1}} \rho_{1,c_1,c_2}\ldots\sum_{j\in\mathcal{S}_{L-1,c_{L-1}}} \rho_{L-1, c_{L-1},j} \nabla F_j(w_j^{(t,k)})\right\rangle
        \leq
        -\frac{\eta^t K^t}{2}\left\|\nabla F(w^{(t)})\right\|^2 \\
        &-\frac{\eta^t}{2} \sum_{k=1}^{K^t}\left\|\sum_{c_1 = 1}^{N_1} \rho_{c_1} \sum_{c_2\in \mathcal{S}_{1,c_1}} \rho_{1,c_1,c_2}\ldots\sum_{j\in\mathcal{S}_{L-1,c_{L-1}}} \rho_{L-1, c_{L-1},j} \nabla F_j(w_j^{(t,k)})\right\|^2 + (\eta^t)^3(K^t)^3\beta^2G^2\\
        &+\frac{(\eta^t)^2 L\beta^2 M (K^t)^5 T q^2 G^2\log(1/\delta)}{\epsilon^2}\sum_{l=2}^{L}(1 - p_{l-1}^{\min})^2 \bigg(p_{l}^{\max}\frac{\alpha_l^2}{\prod_{l' = l}^{L-1} s_{l'}^2} \\
        &+ \sum_{m = l}^{L-1} \prod_{l' = l}^{m-1}(1 - p_{l'}^{\min})p_{m}^{\max}\frac{\alpha_m^2}{\prod_{l' = l}^{ m-1} s_{l'}\prod_{l'' = m}^{ L-1} s_{l''}^2} + \prod_{l' = l}^{ L-1}(1 - p_{l'}^{\min})\frac{\alpha_L^2}{\prod_{l' = l}^{ L-1} s_{l'}}\bigg)
    \end{align*}
\end{lemma}

\begin{proof}  Since $-2 \mathbf a^\top \mathbf b = -\Vert \mathbf a\Vert^2-\Vert \mathbf b \Vert^2 + \Vert \mathbf a- \mathbf b\Vert^2$ holds for any two vectors $\mathbf a$ and $\mathbf b$ with real elements, we have

\begin{align}
    &- \eta^t \sum_{k=1}^{K^t} \mathbb{E}\left\langle \nabla F(w^{(t)}),  \sum_{c_1 = 1}^{N_1} \rho_{c_1} \sum_{c_2\in \mathcal{S}_{1,c_1}} \rho_{1,c_1,c_2}\ldots\sum_{j\in\mathcal{S}_{L-1,c_{L-1}}} \rho_{L-1, c_{L-1},j} \nabla F_j(w_j^{(t,k)})\right\rangle\\
    &= \frac{\eta^t}{2} \sum_{k=1}^{K^t} \Bigg[ -\left\|\nabla F(w^{(t)})\right\|^2 - \left\|\sum_{c_1 = 1}^{N_1} \rho_{c_1} \sum_{c_2\in \mathcal{S}_{1,c_1}} \rho_{1,c_1,c_2}\ldots\sum_{j\in\mathcal{S}_{L-1,c_{L-1}}} \rho_{L-1, c_{L-1},j} \nabla F_j(w_j^{(t,k)})\right\|^2 \\
    &+ \underbrace{\left\|\nabla F(w^{(t)})-\sum_{c_1 = 1}^{N_1} \rho_{c_1} \sum_{c_2\in \mathcal{S}_{1,c_1}} \rho_{1,c_1,c_2}\ldots\sum_{j\in\mathcal{S}_{L-1,c_{L-1}}} \rho_{L-1, c_{L-1},j} \nabla F_j(w_j^{(t,k)})\right\|^2}_\text{(a)}\Bigg]
\end{align}
Applying Assumption~\ref{assump:genLoss}, we can further bound (a) above as:
\begin{align}
    &\left\|\nabla F(w^{(t)})-\sum_{c_1 = 1}^{N_1} \rho_{c_1} \sum_{c_2\in \mathcal{S}_{1,c_1}} \rho_{1,c_1,c_2}\ldots\sum_{j\in\mathcal{S}_{L-1,c_{L-1}}} \rho_{L-1, c_{L-1},j} \nabla F_j(w_j^{(t,k)})\right\|^2\\
    \leq& \sum_{c_1 = 1}^{N_1} \rho_{c_1} \sum_{c_2\in \mathcal{S}_{1,c_1}} \rho_{1,c_1,c_2}\ldots\sum_{j\in\mathcal{S}_{L-1,c_{L-1}}} \rho_{L-1, c_{L-1},j} \left\| \nabla F(w^{(t)}) - \nabla F(w_j^{(t,k)}) \right\|^2\\
    \leq& \beta^2\sum_{c_1 = 1}^{N_1} \rho_{c_1} \sum_{c_2\in \mathcal{S}_{1,c_1}} \rho_{1,c_1,c_2}\ldots\sum_{j\in\mathcal{S}_{L-1,c_{L-1}}} \rho_{L-1, c_{L-1},j} \left\| w^{(t)} - w_j^{(t,k)} \right\|^2
\end{align}
Now for each layer $l$, we use the notation $d_{l,j}$ to denote the parent node of a leaf node $j \in \mathcal{S}_L$:
\begin{align}
&\left\|\nabla F(w^{(t)})-\sum_{c_1 = 1}^{N_1} \rho_{c_1} \sum_{c_2\in \mathcal{S}_{1,c_1}} \rho_{1,c_1,c_2}\ldots\sum_{j\in\mathcal{S}_{L-1,c_{L-1}}} \rho_{L-1, c_{L-1},j} \nabla F_j(w_j^{(t,k)})\right\|^2\\
    \leq& \beta^2\sum_{c_1 = 1}^{N_1} \rho_{c_1} \sum_{c_2\in \mathcal{S}_{1,c_1}} \rho_{1,c_1,c_2}\ldots\sum_{j\in\mathcal{S}_{L-1,c_{L-1}}} \rho_{L-1, c_{L-1},j} \mathbb{E}\left\| -\eta^t\sum_{k'=1}^{k-1} \widehat{g}_j^{(t,k')} + \sum_{l=1}^{L-1} (1 - p_l^{\min}) \sum_{\substack{k'\in \mathcal{K}_l^t }}  \sum_{i\in \mathcal{S}_{l,d_{l,j}}} \rho_{l,c,i}n_{l+1,i}^{(t,k')} \right\|^2\\
    \leq& 2\beta^2(\eta^t)^2  \sum_{c_1 = 1}^{N_1} \rho_{c_1} \sum_{c_2\in \mathcal{S}_{1,c_1}} \rho_{1,c_1,c_2}\ldots\sum_{j\in\mathcal{S}_{L-1,c_{L-1}}} \rho_{L-1, c_{L-1},j} \mathbb{E}\left\| \sum_{k'=1}^{k-1} \widehat{g}_j^{(t,k')} \right\|^2\\
    &+ 2L\beta^2 \sum_{c_1 = 1}^{N_1} \rho_{c_1} \sum_{c_2\in \mathcal{S}_{1,c_1}} \rho_{1,c_1,c_2}\ldots\sum_{j\in\mathcal{S}_{L-1,c_{L-1}}} \rho_{L-1, c_{L-1},j} \mathbb{E}\left\|\sum_{l=1}^{L-1} (1 - p_l^{\min})\sum_{\substack{k'\in \mathcal{K}_l^t }}  \sum_{i\in \mathcal{S}_{l,d_{l,j}}} \rho_{l,c,i}n_{l+1,i}^{(t,k')} \right\|^2
\end{align}

We can then use the Lemma~\ref{lem:bd_on_DP_noise}, and Assumption~\ref{assump:genLoss} to show that
\begin{align}
    &\left\|\nabla F(w^{(t)})-\sum_{c_1 = 1}^{N_1} \rho_{c_1} \sum_{c_2\in \mathcal{S}_{1,c_1}} \rho_{1,c_1,c_2}\ldots\sum_{j\in\mathcal{S}_{L-1,c_{L-1}}} \rho_{L-1, c_{L-1},j} \nabla F_j(w_j^{(t,k)})\right\|^2\\
    \leq & 2\beta^2(K^t)^2 G^2 \\&+  \frac{4  \eta^t L\beta^2 M (K^t)^4 T q^2 G^2\log(1/\delta)}{\epsilon^2}\sum_{l=2}^{L} (1 - p_{l-1}^{\min})^2\bigg(p_{l}^{\max}\frac{\alpha_l^2}{\prod_{l' = l}^{L-1} s_{l'}^2} + \sum_{m = l}^{L-1} \prod_{l' = l}^{m-1}(1 - p_{l'}^{\min})p_{m}^{\max}\frac{\alpha_m^2}{\prod_{l' = l}^{ m-1} s_{l'}\prod_{l'' = m}^{ L-1} s_{l''}^2} \\
    &+ \prod_{l' = l}^{ L-1}(1 - p_{l'}^{\min})\frac{\alpha_L^2}{\prod_{l' = l}^{ L-1} s_{l'}}\bigg)
\end{align}

\end{proof}

\begin{fact}\label{fact:1} 
Consider $n$ random real-valued vectors $\mathbf x_1,\cdots,\mathbf x_n\in\mathbb R^m$, the following inequality holds: 
 \begin{equation}
     \sqrt{\mathbb E\left[\Big\Vert\sum\limits_{i=1}^{n} \mathbf x_i\Big\Vert^2\right]}\leq \sum\limits_{i=1}^{n} \sqrt{\mathbb E[\Vert\mathbf x_i\Vert^2]}.
 \end{equation}
\end{fact}
\begin{proof} Note that
    \begin{align}
        &\sqrt{\mathbb E\left[\Big\Vert\sum\limits_{i=1}^{n}\mathbf x_i\Big\Vert^2\right]}
        =
        \sqrt{\sum\limits_{i,j=1}^{n}\mathbb E [\mathbf x_i^\top\mathbf x_j]}
        \overset{(a)}{\leq}
\sum\limits_{i,j=1}^{n}\sqrt{\mathbb E [\Vert\mathbf x_i\Vert^2] \mathbb E[\Vert\mathbf x_j\Vert^2]]}
        =
        \sum\limits_{i=1}^{n} \sqrt{\mathbb E[\Vert\mathbf x_i\Vert^2]},
    \end{align}
    where $(a)$ follows from Holder's inequality, $\mathbb E[|XY|] \leq \sqrt{\mathbb E[|X|^2]\mathbb E[ |Y|^2]}$.
\end{proof}

\begin{fact} \label{fact:2} Let $a\geq0 $, $b\geq 0$ and $n\geq 1$ (or $n\geq0$ if $n$ integer). Then, it follows
$a^n - b^n \leq (a-b)na^{n-1}$.
\begin{proof}
Let $\phi(x)\triangleq a^n-(a+x)^n$. Since $\phi(x)$ is a concave function of
$x\geq -a$, it follows that $a^n-b^n=\phi(b-a)\leq \phi(0)+\phi^\prime(0)(b-a)
=(a-b)na^{n-1}$.
\end{proof}
\end{fact}

\end{document}